\documentclass{article}

\usepackage[T1]{fontenc}
\usepackage{iclr2027_conference,times}
\iclrfinalcopy
\usepackage{amsmath,amssymb,amsthm,mathtools}
\usepackage{booktabs,microtype,url,xcolor,graphicx,algorithm,algorithmic}
\usepackage[colorlinks=true,allcolors=blue,hypertexnames=false]{hyperref}

\newtheorem{theorem}{Theorem}
\newtheorem{lemma}[theorem]{Lemma}

\theoremstyle{definition}

\theoremstyle{remark}

\newcommand{\E}{\mathbb E}

\newcommand{\DeltaK}{\Delta_K}
\newcommand{\Dir}{\operatorname{Dir}}
\newcommand{\TV}{\operatorname{TV}}
\newcommand{\Reg}{\operatorname{Reg}}
\newcommand{\PUCal}{\operatorname{PUCal}}

\title{Dirichlet Follow-the-Leader Closes the Gap\\in Simultaneous Multiclass U-Calibration}
\author{\rule{0pt}{34pt}\makebox[\dimexpr\linewidth-2\tabcolsep\relax][c]{\begin{tabular}{c}{\fontseries{bx}\selectfont Pahan Dewasurendra}\\{\normalfont Johns Hopkins University}\end{tabular}}}

\begin{document}
\maketitle
\fancyhead{}

\begin{abstract}
Can one forecaster attain the optimal regret rate for every bounded proper loss and also adapt to every smooth proper loss?  Recent work answered this up to a dimension gap.  Its self-concordant perturbation gives roughly $K^{5/4}\sqrt T$ worst-case regret and incurs an additional $\beta\sqrt K\log K$ for $\beta$-smooth losses.  We close both gaps with a one-line forecaster.  After observing class counts $c_{t-1}$, draw the next prediction from $\Dir(c_{t-1})$, on the face of classes seen so far.  This is a fresh Bayesian bootstrap of the outcomes.  The analysis rests on an exact identity: averaging any bounded proper loss under $\Dir(\alpha)$ equals a discrete derivative of its Dirichlet-averaged Bayes risk.  The identity makes the be-the-perturbed-leader term telescope to a nonpositive Jensen gap.  A one-count likelihood ratio then bounds stability by the inverse square root of that class's count.  The resulting single, horizon-free algorithm satisfies
\begin{align*}
 \sup_{\ell}\E\Reg_\ell&\leq4\sqrt{S_TT}\leq4\sqrt{KT},\\
 \E\Reg_\ell&\leq\tfrac52\beta(1+\log T)
 \quad\text{for every $\beta$-smooth proper $\ell$.}
\end{align*}
Here $S_T$ is the number of observed classes.  Known lower bounds show that both rates are optimal in their nontrivial regimes.  The proof covers nondifferentiable losses and changes of the active simplex face.
\end{abstract}

\section{Introduction}

A probabilistic forecast is often consumed by an agent whose downstream utility is unknown to the forecaster.  Proper losses capture this uncertainty: each proper loss corresponds to a way of valuing probabilistic predictions, and truthful prediction minimizes expected loss.  U-calibration asks for one online sequence of forecasts with low regret under every bounded proper loss \citep{kleinberg2023ucal}.  This requirement is substantially stronger than optimizing a fixed score such as Brier loss.

For $K$ outcomes, \citet{luo2024optimal} established the optimal $\Theta(\sqrt{KT})$ worst-case pseudo-U-calibration rate.  Their optimal forecaster does not adapt to easier scores.  In particular, it can incur $\Omega(\sqrt T)$ regret on squared loss even though Follow-the-Leader (FTL) has logarithmic regret.  Very recent work showed that this conflict is not fundamental \citep{frongillo2026simultaneous}.  A carefully shaped self-concordant perturbation simultaneously gives logarithmic regret on smooth scores and sublinear regret on all proper scores.  Its bounds leave two linked gaps.  The general rate is $\widetilde O(K^{5/4}\sqrt T)$ rather than $O(\sqrt{KT})$, and the smooth rate contains an additive $O(\beta\sqrt K\log K)$ term.  The paper explicitly asks whether simultaneous optimality is possible.

We answer this question affirmatively.  Let $c_t=\sum_{s\leq t}y_s$ be the vector of outcome counts.  At time $t\geq2$, independently sample
\[
 P_t\sim\Dir(c_{t-1}).
\]
Zero count coordinates are fixed at zero, so the distribution lives on the face spanned by classes observed so far.  Equivalently, assign independent Gamma weights to the observed outcomes and normalize.  Thus the algorithm is exactly a fresh Bayesian bootstrap \citep{rubin1981bootstrap}.  It requires no horizon, smoothness parameter, learning rate, or optimization oracle. The Bayesian-bootstrap draw itself is classical.  Our claims concern its new adversarial proper-loss analysis and simultaneous regret guarantees.

The main proof is not a generic posterior-sampling argument.  If $f$ is the concave Bayes risk of a proper loss and $F(\alpha)=\E_{P\sim\Dir(\alpha)}f(P)$, we prove
\begin{equation}
 \E_{P\sim\Dir(\alpha)}\ell(P,e_j)
 =nF(\alpha)-(n-1)F(\alpha-e_j),
 \qquad n=\sum_i\alpha_i.                 \tag{1}\label{eq:key-intro}
\end{equation}
This identity applies to arbitrary bounded proper losses, including polyhedral and V-shaped losses.  It turns the perturbed-leader contribution into $T(F(c_T)-f(c_T/T))\leq0$.  The remaining stability term compares two Dirichlet laws whose parameters differ by one count.  Their likelihood ratio is linear in one coordinate, which gives stability $O(1/\sqrt m)$ when that coordinate has appeared $m$ times.  Summing separately within each class produces $O(\sum_i\sqrt{N_i})=O(\sqrt{KT})$.

The same distribution is automatically adapted to smooth losses.  Its mean is empirical FTL and its squared radius is at most $1/t$.  Centering removes the first-order smoothness term, leaving $O(\beta/t)$.  A short semiconcavity argument also gives an explicit $2\beta H_T$ FTL bound under the one-sided smoothness definition used in prior work.

Our contributions are:
\begin{itemize}
\item We give one efficient, horizon-free forecaster with $4\sqrt{S_TT}\leq4\sqrt{KT}$ expected regret for every bounded proper loss and simultaneous $O(\beta\log T)$ regret for every bounded $\beta$-smooth proper loss.
\item We prove the count-deletion identity in \eqref{eq:key-intro}, including nondifferentiable Bayes risks, zero count coordinates, and the singular case $\alpha_j=1$.
\item We isolate two exact geometric facts behind simultaneous adaptation: one-count Dirichlet total variation controls nonsmooth loss, while centered covariance controls smooth loss.
\end{itemize}

We emphasize the quantifier.  Our result controls $\sup_\ell\E\Reg_\ell$, called pseudo-U-calibration.  It does not claim the stronger $\E\sup_\ell\Reg_\ell$ for the infinite class of all proper losses. This is the same expected-regret criterion studied by \citet{frongillo2026simultaneous}.

\section{Setting and algorithm}\label{sec:setting}

Let $e_1,\ldots,e_K$ denote the vertices of $\DeltaK$.  At round $t$, a forecaster chooses a possibly random $P_t\in\DeltaK$, then observes an outcome $y_t\in\{e_1,\ldots,e_K\}$.  We first state results for an oblivious outcome sequence.  Appendix~\ref{app:adaptive} gives the standard extension to a nonanticipating adaptive adversary when fresh randomness is used each round.

A loss $\ell:\DeltaK\times\{e_i\}_{i=1}^K\to[-1,1]$ is \emph{proper} if
\[
 p\in\arg\min_{q\in\DeltaK}\E_{Y\sim p}\ell(q,Y)
 \quad\text{for every }p\in\DeltaK.
\]
For an outcome sequence, propriety makes its empirical distribution $q_T=T^{-1}\sum_ty_t$ a hindsight minimizer.  Define
\[
 \Reg_\ell(T)=\E\sum_{t=1}^T\ell(P_t,y_t)-T f(q_T),
 \qquad f(p)=\sum_i p_i\ell(p,e_i).
\]
The expectation is over the forecaster.  Let $\mathcal L$ be all proper losses with range in $[-1,1]$ and let $\PUCal_T=\sup_{\ell\in\mathcal L}\Reg_\ell(T)$.

Following \citet{frongillo2026simultaneous}, a differentiable loss is $\beta$-smooth when
\begin{equation}
 \ell(q,y)-\ell(p,y)
 \leq\langle\nabla_p\ell(p,y),q-p\rangle
       +\frac\beta2\|q-p\|_2^2                 \tag{2}\label{eq:smooth}
\end{equation}
for every $p,q\in\DeltaK$ and outcome $y$. Gradients and differentiability on the closed simplex are understood relative to its affine hull, including one-sided differentiability on boundary faces.

For a nonnegative parameter vector $\alpha$, write $A=\{i:\alpha_i>0\}$. We define $\Dir(\alpha)$ on the face $\Delta_A$ by drawing independent $G_i\sim\mathrm{Gamma}(\alpha_i,1)$ for $i\in A$, setting $P_i=G_i/\sum_{k\in A}G_k$, and setting $P_i=0$ outside $A$.

\begin{algorithm}[t]
\caption{Dirichlet Follow-the-Leader}
\label{alg:dftl}
\begin{algorithmic}[1]
\STATE Initialize $c_0=0$ and predict $P_1=(1/K,\ldots,1/K)$.
\FOR{$t=1,2,\ldots,T$}
  \IF{$t\geq2$}
    \STATE Draw $P_t\sim\Dir(c_{t-1})$ independently of past draws.
  \ENDIF
  \STATE Observe $y_t$ and set $c_t=c_{t-1}+y_t$.
\ENDFOR
\end{algorithmic}
\end{algorithm}

\begin{theorem}[Simultaneously optimal regret]\label{thm:main}
For every $K,T\geq1$, let $S_T$ be the number of distinct outcomes observed. Algorithm~\ref{alg:dftl} satisfies
\[
 \PUCal_T\leq4\sqrt{S_TT}\leq4\sqrt{\min\{K,T\}T}.
\]
Simultaneously, for every $\beta$-smooth $\ell\in\mathcal L$,
\[
 \Reg_\ell(T)\leq\frac52\beta H_T
 \leq\frac52\beta(1+\log T),
\]
where $H_T=\sum_{t=1}^T1/t$.
\end{theorem}

For $K\leq T$, the $\sqrt{KT}$ worst-case rate is matched by the multiclass V-shaped lower bound of \citet{luo2024optimal}.  Squared loss gives an $\Omega(\log T)$ lower bound for the smooth subclass.  Hence both worst-case rates in Theorem~\ref{thm:main} are optimal up to universal constants in their nontrivial regimes.  The support-sensitive refinement is an additional instance-dependent guarantee.

\section{A Dirichlet identity for every proper loss}\label{sec:identity}

Write $\ell_p=(\ell(p,e_1),\ldots,\ell(p,e_K))$.  Propriety gives
\begin{equation}
 f(q)\leq\langle q,\ell_p\rangle,
 \qquad f(p)=\langle p,\ell_p\rangle.          \tag{3}\label{eq:supergrad}
\end{equation}
Thus $f$ is concave and $\ell_p$ is a supergradient representation.  Since $\|\ell_p\|_\infty\leq1$, applying \eqref{eq:supergrad} in both directions also gives
\begin{equation}
 |f(p)-f(q)|\leq\|p-q\|_1.                    \tag{4}\label{eq:lipschitz-f}
\end{equation}

\begin{lemma}[Count deletion]\label{lem:identity}
Let $\alpha\geq0$, let $n=\sum_i\alpha_i$, and suppose $\alpha_j\geq1$. For $F(\alpha)=\E_{P\sim\Dir(\alpha)}f(P)$,
\[
 \E_{P\sim\Dir(\alpha)}\ell(P,e_j)
 =nF(\alpha)-(n-1)F(\alpha-e_j).
\]
The statement uses the active-face convention above.  When $n=1$, the second term is zero.
\end{lemma}

\begin{proof}
First suppose $\alpha_j>1$ and work relative to the active face.  A concave function is differentiable almost everywhere on the relative interior. At each such $p$, \eqref{eq:supergrad} makes the tangent component of $\ell_p$ the unique supergradient of $f$.  Therefore
\begin{equation}
 \ell(p,e_j)=f(p)+D_{e_j-p}f(p)                \tag{5}\label{eq:savage-dir}
\end{equation}
almost everywhere.

Let $\rho_\alpha$ be the Dirichlet density and set $v(p)=e_j-p$.  A direct calculation on the active face gives
\[
 \operatorname{div}v+D_v\log\rho_\alpha
 =\frac{\alpha_j-1}{p_j}-(n-1).
\]
The boundary flux vanishes.  At a face $p_i=0$ with $i\neq j$, the normal component $v_i=-p_i$ cancels the possible density singularity.  At $p_j=0$, the assumption $\alpha_j>1$ makes the flux vanish.  Integration by parts and \eqref{eq:savage-dir} now yield
\begin{align*}
 \E_\alpha D_vf(P)
 &=(n-1)F(\alpha)
   -(\alpha_j-1)\E_\alpha\frac{f(P)}{P_j},\\
 (\alpha_j-1)\E_\alpha\frac{f(P)}{P_j}
 &=(n-1)F(\alpha-e_j),
\end{align*}
where the second equality is the ratio of Dirichlet normalizers.  Combining the two displays proves the claim for $\alpha_j>1$.

If $\alpha_j=1$, apply the proved identity to $\alpha+\varepsilon e_j$ and let $\varepsilon\downarrow0$.  The first Dirichlet law converges in total variation, which is enough for the possibly discontinuous bounded function $\ell(\cdot,e_j)$.  After deleting one count, the law with parameter $\varepsilon$ in coordinate $j$ converges weakly to the lower-dimensional face.  Equation~\eqref{eq:lipschitz-f} makes $f$ continuous, so its expectations converge.  The case $n=1$ is immediate from $\ell(e_j,e_j)=f(e_j)$.
\end{proof}

The distinction between total variation and weak convergence in the last paragraph matters.  A bounded proper loss can jump across a decision boundary, while its Bayes risk remains continuous.

\section{Worst-case analysis}\label{sec:worst}

For analysis, after observing round $t$, draw a ghost prediction $Q_{t+1}\sim\Dir(c_t)$.  For any fixed loss,
\begin{align}
 \Reg_\ell(T)
 &=\sum_{t=1}^T\E[\ell(P_t,y_t)-\ell(Q_{t+1},y_t)] \notag\\
 &\quad+\sum_{t=1}^T\E\ell(Q_{t+1},y_t)-Tf(q_T).  \tag{6}\label{eq:decomp}
\end{align}
If $y_t=e_j$, Lemma~\ref{lem:identity} gives
\[
 \E\ell(Q_{t+1},y_t)=tF(c_t)-(t-1)F(c_{t-1}).
\]
The second line of \eqref{eq:decomp} therefore telescopes to
\[
 T\{F(c_T)-f(q_T)\}\leq0,                     \tag{7}\label{eq:jensen}
\]
because $\E Q_{T+1}=q_T$ and $f$ is concave.

It remains to control stability.  The next elementary estimate is where the coordinate adaptivity enters.

\begin{lemma}[One-count stability]\label{lem:tv}
Let $c\geq0$, $n=\sum_i c_i$, and $c_j=m>0$.  Then
\[
 \TV(\Dir(c),\Dir(c+e_j))
 \leq\frac12\sqrt{\frac{n-m}{m(n+1)}}
 \leq\frac1{2\sqrt m}.
\]
\end{lemma}

\begin{proof}
On their common active face, the likelihood ratio is
\[
 \frac{d\Dir(c+e_j)}{d\Dir(c)}(p)=\frac{np_j}{m}.
\]
Consequently,
\[
 \TV(\Dir(c),\Dir(c+e_j))
 =\frac12\E_{P\sim\Dir(c)}\left|\frac{nP_j}{m}-1\right|.
\]
Cauchy--Schwarz and $\operatorname{Var}(P_j)=m(n-m)/(n^2(n+1))$ prove both inequalities. When $b=n-m>0$, direct integration at the likelihood-ratio crossing $p_j=m/n$ also gives the exact expression
\[
 \TV=\frac{(m/n)^m(1-m/n)^b}{mB(m,b)}.
\]
For $b=0$, both laws are the same point mass and the distance is zero.
\end{proof}

If class $j$ has appeared $m>0$ times before round $t$, the range $[-1,1]$ and Lemma~\ref{lem:tv} bound the corresponding stability term by $2\TV\leq1/\sqrt m$.  If it is the first appearance, the supports differ and the trivial bound is $2$.  Let $N_i=c_{T,i}$ and $S=|\{i:N_i>0\}|$.  Equations \eqref{eq:decomp} and \eqref{eq:jensen} give
\begin{align*}
 \Reg_\ell(T)
 &\leq2S+\sum_{i=1}^K\sum_{m=1}^{N_i-1}\frac1{\sqrt m}\\
 &\leq2S+2\sum_{i:N_i>0}\sqrt{N_i}\\
 &\leq2S+2\sqrt{ST}\leq4\sqrt{ST},
\end{align*}
where $S\leq T$ implies $S\leq\sqrt{ST}$. The bound is independent of $\ell$, so taking the supremum proves the first part of Theorem~\ref{thm:main}.

\section{Smooth-loss adaptation}\label{sec:smooth}

For $t\geq2$, Dirichlet moments give
\begin{equation}
 \E P_t=q_{t-1},\qquad
 \E\|P_t-q_{t-1}\|_2^2
 =\frac{1-\|q_{t-1}\|_2^2}{t}\leq\frac1t.     \tag{8}\label{eq:moments}
\end{equation}
Set $q_0=P_1$ for convenience.  Applying \eqref{eq:smooth}, taking expectations, and using centering yields
\begin{equation}
 \E[\ell(P_t,y_t)-\ell(q_{t-1},y_t)]
 \leq\frac{\beta}{2t},\qquad t\geq2.           \tag{9}\label{eq:perturb-smooth}
\end{equation}

For completeness, we prove the FTL estimate under precisely the one-sided definition \eqref{eq:smooth}.  This avoids assuming that \eqref{eq:smooth} implies full gradient Lipschitzness.

\begin{lemma}[Smooth proper-loss FTL]\label{lem:smooth-ftl}
For every differentiable proper loss satisfying \eqref{eq:smooth}, empirical FTL satisfies
\[
 \sum_{t=1}^T\ell(q_{t-1},y_t)-Tf(q_T)\leq2\beta H_T.
\]
\end{lemma}

\begin{proof}
We first establish an endpoint condition from propriety and boundary differentiability.  Fix $i\neq j$ and set $q_\varepsilon=(1-\varepsilon)e_j+\varepsilon e_i$.  Properness at $q_\varepsilon$, comparing the truthful prediction with $e_j$, says
\[
 (1-\varepsilon)\ell(q_\varepsilon,e_j)
 +\varepsilon\ell(q_\varepsilon,e_i)
 \leq(1-\varepsilon)\ell(e_j,e_j)
 +\varepsilon\ell(e_j,e_i).
\]
Expanding at zero gives $D_{e_i-e_j}\ell(e_j,e_j)\leq0$.  Propriety at $e_j$ also says that $e_j$ globally minimizes $\ell(\cdot,e_j)$, so the reverse inequality holds. Thus every tangent derivative of $\ell(\cdot,e_j)$ at $e_j$ is zero.

Fix $p\in\DeltaK$, write $d=\|e_j-p\|_2$, and define $\phi(s)=\ell((1-s)p+se_j,e_j)$.  Equation~\eqref{eq:smooth} is equivalent to concavity of $h(s)=\phi(s)-(\beta d^2/2)s^2$.  Since $\phi'(1)=0$, monotonicity of the derivative of $h$ gives, for $s<1$,
\[
 \phi'(s)-\beta d^2s=h'(s)\geq h'(1)=-\beta d^2.
\]
Hence $\phi'(s)\geq-\beta d^2(1-s)$.  The FTL update after outcome $e_j$ is $q_t=(1-1/t)q_{t-1}+e_j/t$.  Integrating from $0$ to $1/t$ gives
\[
 \ell(q_{t-1},e_j)-\ell(q_t,e_j)
 \leq\frac{\beta\|e_j-q_{t-1}\|_2^2}{t}
 \leq\frac{2\beta}{t}.
\]
The standard be-the-leader inequality $\sum_t\ell(q_t,y_t)\leq Tf(q_T)$ and summation prove the result.
\end{proof}

Adding \eqref{eq:perturb-smooth} to Lemma~\ref{lem:smooth-ftl} gives at most $2\beta H_T+(\beta/2)(H_T-1)\leq(5\beta/2)H_T$, completing Theorem~\ref{thm:main}.

\section{Related work and interpretation}\label{sec:related}

\paragraph{U-calibration.} \citet{kleinberg2023ucal} introduced U-calibration for binary outcomes and connected it to unknown downstream agents.  \citet{luo2024optimal} obtained the optimal $\Theta(\sqrt{KT})$ multiclass pseudo-U-calibration rate, proved the V-shaped lower bound, and identified loss classes on which FTL is logarithmic.  Their worst-case-optimal perturbed leader is not adaptive to smooth losses.  \citet{frongillo2026simultaneous} introduced a self-concordant prediction-space perturbation that adapts to smooth and log-barrier-smooth losses.  For the classes considered here, its bounds are approximately $K^{5/4}\sqrt T$ and $\beta\log T+\beta\sqrt K\log K$.  Our result removes these dimension gaps. Calibeating, omniprediction, and swap-regret results study related universal downstream guarantees with different benchmarks \citep{roth2024forecasting,chen2026calibeating}.

\paragraph{Perturbation and posterior sampling.} Follow-the-Perturbed-Leader is classical \citep{kalai2005efficient,cesabianchi2006prediction}.  Fresh perturbations also yield the standard nonanticipating adaptive-adversary extension \citep{hutter2005adaptive}.  Algorithm~\ref{alg:dftl} can instead be viewed as posterior sampling for a categorical model under the limiting zero-mass Dirichlet prior, or as Rubin's Bayesian bootstrap.  Dirichlet-weighted loss minimization appears in the loss-likelihood bootstrap \citep{lyddon2019general}, and normalized random weights have also been maintained in online Bayesian bagging \citep{lee2004lossless}.  Dirichlet Bayes mixtures are classical in universal coding under log loss \citep{watanabe2013achievability}, while size-biased Dirichlet integral identities are classical probability machinery \citep{last2020integral}. None of these works gives adversarial regret for an unknown proper loss.  Our contribution is the loss-uniform count-deletion specialization, its telescoping use, and the simultaneously optimal guarantees, not the bootstrap draw itself.

\paragraph{Why Dirichlet geometry fits both regimes.} The update $c\mapsto c+e_j$ has likelihood ratio $np_j/c_j$, so its total variation automatically scales with the number of previous appearances of the updated class.  This is the right geometry for discontinuous proper losses.  At the same time,
\[
 \operatorname{Cov}(P_t)
 =\frac{\operatorname{diag}(q_{t-1})-q_{t-1}q_{t-1}^{\top}}{t},
\]
which has the centered $O(1/t)$ radius needed for smooth losses.  The exact identity in Lemma~\ref{lem:identity} is what lets these two local facts control global regret without a separate perturbed-leader path-length penalty.

\section{Limitations and conclusion}

Our theorem resolves simultaneous optimality for bounded proper losses and the smooth subclass under the expected-regret criterion.  It does not control actual U-calibration $\E\sup_{\ell\in\mathcal L}\Reg_\ell$ over the entire infinite class.  Moving the supremum inside expectation requires complexity control or a different argument.  We also do not address the broader class of losses smooth relative to a log barrier, where Euclidean Dirichlet covariance need not be the right measure.

The result shows that the previous tradeoff was geometric rather than information-theoretic.  A Bayesian-bootstrap sample of empirical FTL has exactly the count stability needed by nonsmooth scores and exactly the centering needed by smooth scores.  The count-deletion identity makes those two properties compatible in one short analysis.

\bibliography{references}
\bibliographystyle{iclr2027_conference}

\appendix

\section*{AI use statement}
Generative AI tools were used to assist with proofs and language editing.

\section{Adaptive outcomes}\label{app:adaptive}

The main text assumes a fixed outcome sequence.  The same expected bounds hold for a nonanticipating adaptive adversary that may depend on past predictions but not on the fresh draw $P_t$.  Condition on the history before the round and on the selected outcome.  Given the current count vector, $P_t$ is a fresh Dirichlet draw, so the one-step identity and stability bounds hold conditionally.  Summing conditional expectations gives the same telescoping count potential along the realized count path.  This is the usual fresh-randomization reduction for perturbed leader \citep{hutter2005adaptive}.  No guarantee is possible against an adversary that observes the current randomized forecast before choosing the same-round outcome under this protocol.

\section{Additional boundary details for Lemma~\ref{lem:identity}}

We record a standard approximation that makes the nonsmooth integration by parts fully formal.  Restrict $f$ to the relative interior of the active face. It is concave and Lipschitz by \eqref{eq:lipschitz-f}.  Convolve it on compact subsets with a smooth mollifier, apply integration by parts to the smooth approximants, and then exhaust the face.  The directional derivatives of a concave Lipschitz function converge almost everywhere and are uniformly bounded in tangent directions.  Dominated convergence applies to every term when $\alpha_j>1$.  The boundary flux estimate in the main proof is uniform under this exhaustion.

For $\alpha_j=1$, write $\alpha^{(\varepsilon)}=\alpha+\varepsilon e_j$.  On the common active face, the densities of $\Dir(\alpha^{(\varepsilon)})$ converge in $L^1$ to that of $\Dir(\alpha)$, hence in total variation.  After deleting $e_j$, the marginal coordinate $P_j$ has a Beta distribution with first parameter $\varepsilon$ and therefore converges to zero in probability.  Conditional proportions of the remaining coordinates have distribution $\Dir(\alpha_{-j})$.  This proves weak convergence to the deleted face.  Boundedness handles the left side and continuity of $f$ handles the right side.

\section{Sanity checks}

A supplementary script evaluates the identity for Brier, spherical, binary V-shaped, and multiclass polyhedral proper losses.  It also checks the likelihood-ratio total-variation bound and exhaustively enumerates all binary outcome sequences through $T=8$.  The smooth Brier identities are also checked in closed form.  These calculations are not used by any proof.

\end{document}